\documentclass[twoside]{article}
\usepackage[accepted]{aistats2020}

\usepackage[round]{natbib}

\usepackage[utf8]{inputenc} 
\usepackage[T1]{fontenc}    
\usepackage{hyperref}       
\usepackage{url}            
\usepackage{booktabs}       
\usepackage{amsfonts}       
\usepackage{nicefrac}       
\usepackage{microtype}      
\usepackage{amsmath}
\usepackage{amsthm}
\usepackage{color}
\usepackage{enumitem}

\usepackage{algorithm}
\usepackage[noend]{algpseudocode}
\usepackage{caption}
\usepackage{subcaption}

\usepackage{tikz}
\usetikzlibrary{calc,arrows}

\newtheorem{theorem}{Theorem}[section]

\newtheorem{definition}[theorem]{Definition}
\newtheorem{lemma}[theorem]{Lemma}

\newtheorem{assumption}[theorem]{Assumption}
\newtheorem{remark}[theorem]{Remark}

\newcommand{\BEAS}{\begin{eqnarray*}}
\newcommand{\EEAS}{\end{eqnarray*}}
\newcommand{\BEA}{\begin{eqnarray}}
\newcommand{\EEA}{\end{eqnarray}}
\newcommand{\BEQ}{\begin{equation}}
\newcommand{\EEQ}{\end{equation}}
\newcommand{\BIT}{\begin{itemize}}
\newcommand{\EIT}{\end{itemize}}
\newcommand{\BNUM}{\begin{enumerate}}
\newcommand{\ENUM}{\end{enumerate}}

\newcommand{\BA}{\begin{array}}
\newcommand{\EA}{\end{array}}

\newcommand{\eg}{{\it e.g.}}

\newcommand{\reals}{{\mathbb R}}

\newcommand{\diag}{\mathop{\bf diag}}

\newcommand{\Real}{{\mathbb R}}
\newcommand{\st}{~~~\text{s.t.}~~~}

\newcommand\ellipsebyfoci[4]{
   \path[#1] let \p1=(#2), \p2=(#3), \p3=( $(\p1)!.5!(\p2)$ )
   in \pgfextra{
    \pgfmathsetmacro{\angle}{atan2(\y2-\y1,\x2-\x1)}
    \pgfmathsetmacro{\focal}{veclen(\x2-\x1,\y2-\y1)/2/1cm}
    \pgfmathsetmacro{\lentotcm}{\focal*2*#4}
    \pgfmathsetmacro{\axeone}{(\lentotcm - 2 * \focal)/2+\focal}
    \pgfmathsetmacro{\axetwo}{sqrt((\lentotcm/2)*(\lentotcm/2)-\focal*\focal}
   }
   (\p3) ellipse[x radius=\axeone cm,y radius=\axetwo cm, rotate=\angle];
}

\newcommand\sbullet[1][.5]{\mathbin{\vcenter{\hbox{\scalebox{#1}{$\bullet$}}}}}
\newcommand\Ical{{\mathcal I}}
\newcommand\Iint{\mathring{\mathcal I}}
\newcommand\Xcal{{\mathcal X}}

\runningtitle{Screening Data Points in Empirical Risk Minimization}

\begin{document}

\twocolumn[

\aistatstitle{Screening Data Points in Empirical Risk Minimization \\ via Ellipsoidal Regions and Safe Loss Functions}

\aistatsauthor{Grégoire Mialon \And Alexandre d'Aspremont \And Julien Mairal}

\aistatsaddress{ Inria\footnotemark[1]\textsuperscript{,}\footnotemark[2] \And CNRS, ENS\footnotemark[2] \And Inria\footnotemark[1]}]

\footnotetext[1]{Univ.   Grenoble  Alpes,  Inria,  CNRS,Grenoble  INP,  LJK,  38000  Grenoble, France. <firstname.lastname@inria.fr>.}
\footnotetext[2]{Département d’informatique de l’ENS, CNRS, Inria, PSL, 75005 Paris, France. <aspremon@ens.fr>.}

\begin{abstract}
We design simple screening tests to automatically discard data samples in empirical risk minimization without losing optimization guarantees. We derive loss functions that produce dual objectives with a sparse solution. We also show how to regularize convex losses to ensure such a dual sparsity-inducing property, and propose a general method to design screening tests for classification or regression based on ellipsoidal approximations of the optimal set. In addition to producing computational gains, our approach also allows us to compress a dataset into a subset of representative points.

\end{abstract}

\section{INTRODUCTION}
\label{sec:introduction}
Let us consider a collection of $n$ pairs $(a_i,b_i)_{i=1,\ldots,n}$, where each vector~$a_i$ in~$\Real^p$ describes a data point and~$b_i$ is its label.
For regression, $b_i$ is real-valued, and we address the convex optimization problem
\begin{equation}
   \min_{x \in \Real^p, t \in \Real^n} f(t) + \lambda R(x) \st   t = A x - b, \label{eq:regression}\tag{${\mathcal P}_1$}
\end{equation}
where $A = [a_1,\ldots,a_n]^\top$ in $\Real^{n \times p}$ carries the feature vectors, and $b=[b_1,\ldots,b_n]$ carries the labels. The function $f$ is a convex loss and measures the fit between data points and the model, and $R$ is a convex regularization function.
For classification, the scalars~$b_i$ are binary labels in $\{-1,+1\}$, and we consider instead of~\eqref{eq:regression} margin-based loss functions, where our problem becomes
\begin{equation}
   \min_{x \in \Real^p, t \in \Real^n} f(t) + \lambda R(x) \st   t=\diag(b)A x, \label{eq:classification}\tag{${\mathcal P}_2$}
\end{equation}
The above problems cover a wide variety of formulations such as
Lasso~\citep{tibshirani1996regression} and its
variants~\citep{zou2005regularization}, logistic regression, support vector
machines \citep{friedman2001elements}, and many more. When $R$ is the $\ell_1$-norm, the solution is encouraged to be
sparse~\citep{bach2012optimization}, which can be exploited to speed-up 
optimization procedures. 

A recent line of work has focused on screening tests that seek to automatically discard variables 
before running an optimization algorithm.
For example,~\cite{safe} 
derive a screening rule from Karush-Kuhn-Tucker conditions,
noting that if a dual optimal variable satisfies a given inequality
constraint, the corresponding primal optimal variable must be zero. 
Checking this condition on a set that is known to contain the
optimal dual variable ensures that the corresponding  primal variable can be safely removed.
This prunes
out irrelevant features {\em before} solving the problem. This is called a \textit{safe} rule if it discards variables that are guaranteed to be useless; but it is possible to relax the ``safety'' of the rules~\citep{tibshirani2012strong} without losing too much accuracy in practice. The seminal approach by~\cite{safe} has led to a series of works proposing refined tests~\citep{ellipsoids,wang2013lasso} or dynamic rules~\citep{safer_rules} for the Lasso, where screening is performed as the optimization algorithm proceeds, significantly speeding up convergence. Other papers have proposed screening rules for sparse logistic regression~\citep{reglog_screening} or other linear models.

Whereas the goal of these previous methods is to remove \emph{variables}, our
goal is to design screening tests for \emph{data points} in order to
remove observations that do not contribute to the final model.  The
problem is important when there is a large amount of ``trivial''
observations that are useless for learning. This typically occurs in {\em tracking or 
anomaly detection} applications, where a classical heuristic seeks to mine the data
to find difficult examples~\citep{felzenszwalb2009object}. A few of such screening tests for data points have been proposed in the
literature. Some are problem-specific (\textit{e.g.}~\cite{ogawa2014safe} for SVM), others are making strong assumptions on the objective. For
instance, the most general rule of~\cite{double_screening} for classification
requires strong convexity and the ability to compute a duality gap in closed
form. The goal of our paper is to provide a more generic approach for screening data
samples, both for regression and classification. Such screening tests may be designed for loss functions that induce
a sparse dual solution. We describe this class of loss functions and investigate a regularization mechanism that ensures that the loss enjoys such a property. Our contributions can be summarized as follows:
\begin{itemize}[leftmargin=*]
\item We revisit the Ellipsoid method~\citep{ellipsoids_survey} to design screening test for samples, when the objective is convex and its dual admits a sparse solution.
\item We propose a new regularization mechanism to design regression or classification losses that 
   induce sparsity in the dual. This allows us to recover existing loss functions and to discover new ones with sparsity-inducing properties in the dual.  
\item Originally designed for linear models, we extend our screening rules to kernel methods. Unlike the existing literature, our method also works for non strongly convex objectives.
\item We demonstrate the benefits of our screening rules in various numerical experiments on large-scale classification problems and regression\footnotemark[1]. 
\end{itemize}

\footnotetext[1]{Our code is available at \url{https://github.com/GregoireMialon/screening\_samples}.}
\section{PRELIMINARIES}
\label{sec:tools}
We now present the key concepts used in our paper.

\subsection{Fenchel Conjugacy}

\begin{definition}[Fenchel conjugate]
   Let $f: \Real^p \to \Real \cup \{-\infty,+\infty\}$ be an extended real-valued function. The Fenchel conjugate of $f$ is defined by
\begin{equation*}
    f^*(y) = \underset{t \in \mathbb{R}^p}{\text{max }} \langle t, y \rangle - f(t).
\end{equation*}
\end{definition}
The biconjugate of $f$ is naturally the conjugate of $f^*$ and is denoted by $f^{**}$. The Fenchel-Moreau theorem~\citep{Hiri96} states that if $f$ is proper, lower semi-continuous and convex, then it is equal to its biconjugate $f^{**}$. Finally, Fenchel-Young's inequality gives for all pair $(t,y)$
\begin{equation*}
    f(t) + f^*(y) \geq \langle t, y \rangle,
\end{equation*}
with an equality case iff $y \in \partial f(t)$.

Suppose now that for such a function~$f$, we add a convex term $\Omega$ to $f^*$ in the definition of the biconjugate. We get a modified biconjugate $f_{\mu}$, written
\begin{align*}
     f_{\mu}(t) & = \underset{y \in \mathbb{R}^p}{\text{max }} \langle y, t \rangle - f^*(y) - \mu \Omega (y) \\
     & = \underset{y \in \mathbb{R}^p}{\text{max}} \langle y, t \rangle + \underset{z \in \mathbb{R}^p}{\text{min}} \left\{- \langle z, y \rangle + f(z) \right\} - \mu \Omega (y).
\end{align*}
The inner objective function is continuous, concave in~$y$ and convex in $z$, such that we can switch min and max according to Von Neumann's minimax theorem to get
\begin{align*}
     f_{\mu}(t) & = \underset{z \in \mathbb{R}^p}{\text{min }} f(z) + \underset{y \in \mathbb{R}^p}{\text{max }} \left\{ \langle t - z, y \rangle - \mu \Omega (y) \right\}\\
     & = \underset{z \in \mathbb{R}^p}{\text{min }} f(z) + \mu \Omega^*\left(\frac{t - z}{\mu}\right).
\end{align*}

\begin{definition}[Infimum convolution]
$f_{\mu}$ is called the infimum convolution of $f$ and $\Omega^*$, which may be written as $f ~\square~ \Omega^*$. 
\end{definition}

Note that $f_{\mu}$ is convex as the minimum of a convex function in $(t, z)$. We recover the Moreau-Yosida smoothing~\citep{moreau,yosida} and  its generalization when $\Omega$ is respectively a quadratic term or a strongly-convex term~\citep{Nest03}.  

\subsection{Empirical Risk Minimization and Duality}\label{subsec:dual}

Let us consider the convex ERM problem
\begin{equation}
\label{eq:erm}
   \min_{x \in \Real^p} P(x) = \frac{1}{n} \sum_{i=1}^{n} f_i(a_i^\top x) + \lambda R(x),
\end{equation}
which covers both~(\ref{eq:regression}) and~(\ref{eq:classification}) by using the appropriate definition of function $f_i$.
We consider the dual problem (obtained from Lagrange duality)
\begin{equation}
\label{eq:dual}
    \max_{\nu \in \Real^n} D(\nu) = \frac{1}{n} \sum_{i=1}^n - f_i^*(\nu_i) - \lambda R^*\left(-\frac{A^T \nu}{\lambda n}\right).
\end{equation}
We always have $P(x) \geq D(\nu)$. Since there exists a pair $(x,t)$ such that $Ax=t$ (Slater's conditions), we have $P(x^\star) = D(\nu^\star)$ and $x^\star = -\frac{A^\top \nu^\star}{\lambda n}$ at the optimum.

\subsection{Safe Loss Functions and Sparsity in the Dual of ERM Formulations} A key feature of our losses is to encourage sparsity of dual solutions, which typically emerge from loss functions with a flat region. We call such functions ``safe losses'' since they will allow us to design safe screening tests.

\begin{definition}[Safe loss function]
\label{def:margin_loss}
Let $f: \Real \to \Real$ be a continuous convex loss function such that $\inf_{t \in \mathbb{R}} f(t) = 0$. We say that $f$ is a safe loss if there exists a non-singleton and non-empty interval $\mathcal{I} \subset \mathbb{R}$ such that
\begin{equation*}
    t \in \mathcal{I} \implies f(t) = 0.
\end{equation*} 
\end{definition}

\begin{lemma}[Dual sparsity]\label{lemma:margin}
\label{lemma:margin_sparsity}
   Consider the problem~(\ref{eq:erm}) where $R$ is a convex penalty. Denoting by $x^\star$ and $\nu^\star$ the optimal primal and dual variables respectively, we have for all $i =1,\ldots, n$,
   $$ \nu^\star_i \in \partial f_i(a_i^\top x^\star).$$
\end{lemma}

The proof can be found in Appendix~\ref{sec:add_proofs}.

\begin{remark}[Safe loss and dual sparsity]
A consequence of this lemma is that for both classification and
regression, the sparsity of the dual solution is related to loss functions that have ``flat'' regions---that is, such that $0 \in \partial f_i'(t)$. This is the case for safe loss functions defined above.
\end{remark}

The relation between flat losses and sparse dual solutions is classical, see~\cite{steinwart2004sparseness,blondel19}.
\section{SAFE RULES FOR SCREENING DATA POINTS}
\label{sec:safe_lasso}
In this section, we derive screening rules in the spirit of SAFE~\citep{safe} to select data points in regression or classification problems with safe losses. 

\subsection{Principle of SAFE Rules for Data Points}

We recall that our goal is to safely delete data points prior to optimization, that is, we want to train the model on a subset of the original dataset while still getting the same optimal solution as a model trained on the whole dataset. 
This amounts to identifying beforehand which dual variables are zero at the optimum. Indeed, as discussed in Section~\ref{subsec:dual}, the optimal primal variable $x^\star = - \frac{A^\top \nu^\star}{\lambda n}$ only relies on non-zero entries of $\nu^\star$. 
To that effect, we make the following assumption:
\begin{assumption}[Safe loss assumption]\label{assum:safe}
   We consider problem~(\ref{eq:erm}), where each $f_i$ is a safe loss function. Specifically, we assume that $f_i(a_i^\top x) = \phi( a_i^\top x - b_i)$ for regression, or $f_i(a_i^\top x) = \phi( b_i a_i^\top x)$ for classification, where $\phi$ satisfies Definition~\ref{def:margin_loss} on some interval $\Ical$. For simplicity, we assume that there exists $\mu > 0$ such that $\Ical = [-\mu,\mu]$ for regression losses and $\Ical = [\mu, +\infty)$ for classification, which covers most useful cases.
\end{assumption}
We may now state the basic safe rule for screening.
\begin{lemma}[SAFE rule]
   Under Assumption~\ref{assum:safe}, consider a subset $\Xcal$ containing the optimal solution~$x^\star$. 
   If, for a given data point $(a_i, b_i)$, $ a_i^\top x  - b_i \in \mathring{\mathcal{I}}$ for all $x$ in $\Xcal$, (resp. $b_i  a_i^\top x  \in \mathring{\mathcal{I}}$), where $\mathring{\mathcal{I}}$ is the interior of $\mathcal{I}$,
   then this data point can be discarded from the dataset.
\end{lemma}

\begin{proof}
   From the definition of safe loss functions, $f_i$ is differentiable at $a_i^\top x^\star$ with $\nu_i^\star = f_i'(a_i^\top x) = 0$.
\end{proof}

We see now how the safe screening rule can be interpreted in terms of discrepancy between the model prediction $a_i^\top x$ and the true label $b_i$. If, for a set $\mathcal{X}$ containing the optimal solution $x^*$ and a given data point $(a_i, b_i)$, the prediction always lies in~$\Iint$, then the data point can be discarded from the dataset. The data point screening procedure therefore consists in \textit{maximizing linear forms}, $a_i^\top x - b_i$ and $-a_i^\top x + b_i$ in regression (resp. minimizing $b_i a_i^\top x$ in classification), over a set $\mathcal{X}$ containing $x^*$ and check whether they are lower (resp. greater) than the threshold $\mu$. The smaller $\mathcal{X}$, the lower the maximum (resp. the higher the minimum) hence the more data points we can hope to safely delete. Finding a good test region~$\Xcal$ is critical however. We show how to do this in the next section.

%
\subsection{Building the Test Region \texorpdfstring{$\mathcal{X}$}{X}}

Screening rules aim at sparing computing resources, testing a data point should therefore be easy. As in \cite{safe} for screening variables, if $\mathcal{X}$ is an ellipsoid, the optimization problem detailed above admits a closed-form solution. 
Furthermore, it is possible to get a smaller set $\mathcal{X}$ by adding a first order optimality condition with a subgradient $g$ of the objective evaluated in the center $z$ of this ellipsoid. This linear constraint cuts the final ellipsoid roughly in half thus reducing its volume.

\begin{lemma}[Closed-form screening test]\label{lemma:test}
    Consider the optimization problem
    \BEQ
    \BA{ll}
    \mbox{maximize} & a_i^\top x - b_i\\
    \mbox{subject to} &  (x - z)^T E^{-1} (x - z) \leq 1 \\
    & g^T(x - z) \leq 0
    \EA\EEQ
    in the variable $x$ in $\mathbb{R}^p$ with $E$ defining an ellipsoid with center $z$ and $g$ is in $\mathbb{R}^p$. Then the maximum is
    \begin{equation*}
    \begin{cases}
    a_i^\top z + (a_i^\top E a_i)^{\frac{1}{2}} - b_i \text{ if } g^TE a_i < 0 \\
    a_i^\top \left( z + \frac{1}{2 \gamma} E ( a_i - \nu g ) \right) - b_i \text{ otherwise},
    \end{cases}
    \end{equation*}
    with $ \nu = \frac{g^T E a_i}{g^T E g}$ and $\gamma = \left(\frac{1}{2} (a_i - \nu g)^\top E(a_i - \nu g)\right)^{\frac{1}{2}}$.
    \label{test}
\end{lemma}

The proof can be found in Appendix~\ref{sec:add_proofs} and it is easy to modify it for minimizing $b_i a_i^\top x$. We can obtain both $E$ and $z$ by using a few steps of \textit{the ellipsoid method} \citep{Nemi79,ellipsoids_survey}. This first-order optimization method starts from an initial ellipsoid containing the solution~$x^*$ to a given convex problem (here~\ref{eq:erm}) . It iteratively computes a subgradient in the center of the current ellipsoid, selects the half-ellipsoid containing $x^*$, and computes the ellipsoid with minimal volume containing the previous half-ellipsoid before starting all over again. Such a method, presented in Algorithm~\ref{algo:ell_method}, performs closed-form updates of the ellipsoid. It requires $O(p^2\log(\nicefrac{RL}{\epsilon}))$ iterations for a precision $\epsilon$ starting from a ball of radius $R$ with the Lipschitz bound $L$ on the loss, thus making it impractical for accurately solving high-dimensional problems. Finally, the ellipsoid update formula was also used to screen primal variables for the Lasso problem
\citep{ellipsoids}, although not iterating over ellipsoids in order to get
smaller volumes.

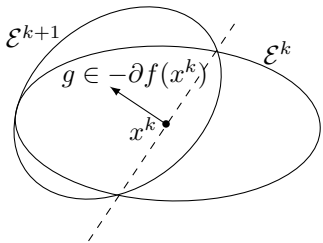
\begin{figure}
    \centering
    \begin{tikzpicture}[scale=0.56]
        \coordinate (a) at (-1.24,1.24);
        \coordinate (b) at (5,1);
        \ellipsebyfoci{draw}{a}{b}{1.16}
        \coordinate (c) at (-0.7661771140981, 0.4732063199958);
        \coordinate (d) at (2.1461475766973, 2.7417539738785);
        \ellipsebyfoci{draw}{c}{d}{1.47}
        \draw[dashed] (0,-1.67) -- (3.5,3.54);
        \draw[latex-] (0.5,2) --  (1.85,1.1);
        \draw (1.1, 2.3) node[scale=1] {$g \in -\partial f(x^k)$};
        \draw (1.3, 1) node[scale=1] {$x^k$};
        \draw (4.5, 2.8) node[scale=1] {$\mathcal{E}^k$};
        \draw (-1.3, 3.2) node[scale=1] {$\mathcal{E}^{k+1}$};
        \node at (1.84,1.11) [circle,fill,inner sep=1pt]{};
    \end{tikzpicture}
    \caption{One step of the ellipsoid method.}
    \label{fig:ellipsoid_method}
\end{figure}


\begin{algorithm}
\caption{Building ellipsoidal test regions}
\label{algo:ell_method}
\begin{algorithmic}[1]
\State \textbf{initialization:} Given $\mathcal{E}^0(x_0, E_0)$ containing $x^*$;
\While{$k < nb_{\text{steps}}$}
   \State $\sbullet$ Compute a subgradient $g$ of~(\ref{eq:erm}) in $x_k$;
  \State $\sbullet$ $\Tilde{g} \gets g / \sqrt{g^T E_k g}$;
  \State $\sbullet$ $x_{k+1} \gets x_{k} - \frac{1}{p+1} E_k \Tilde{g}$;
  \State $\sbullet$ $E_{k+1} \gets \frac{p^2}{p^2 - 1}( E_k - \frac{2}{p + 1} E_k \Tilde{g} \Tilde{g}^T E_k)$;
\EndWhile
\State For regression problems:
\For{each sample $a_i \text{ in } A$}
    \If{${\text{max}} |a_i^\top x - b_i| \leq \mu \text{ for } x \in \mathcal{E}^{nb_{\text{steps}}}$}
    \State Discard $a_i$ from $A$.
    \EndIf
\EndFor
\State For classification, replace condition $|a_i^\top x - b_i| \leq \mu$ by $b_i a_i^\top x \geq \mu$ in the above expression.
\end{algorithmic}
\end{algorithm}

\vspace*{-0.4cm}

\paragraph{Initialization.}
The algorithm requires an initial ellipsoid $\mathcal{E}^0(x_0, E_0)$ that contains the solution. This is
typically achieved by defining the center $x_0$ as an approximate solution of
the problem, which can be obtained in various ways. For instance, one may run a few
steps of a solver on the whole dataset, or one may consider the solution obtained previously
for a different regularization parameter when computing a regularization path, or the solution obtained
for slightly different data, \eg, for tracking applications where an optimization problem has to be solved at every time step $t$, with slight modifications from time $t-1$.

Once the center~$x_0$ is defined, there are many cases where the initial ellipsoid can
be safely assumed to be a sphere. For instance, if the objective---let us call it~$F$---is $\kappa$-strongly convex,
we have the basic inequality $\frac{\kappa}{2}\|x_0-x^\star\|^2 \leq F(x_0)-F^\star$, which can often be 
upper-bounded by several quantities, \eg, a duality gap~\citep{double_screening} or simply $F(x_0)$ if $F$ is non-negative as in typical ERM problems.
Otherwise, other strategies can be used depending on the problem at hand. If the problem is not strongly convex but constrained (\textit{e.g.} often a norm constraint in ERM problems), the initialization is also natural (\textit{e.g.}, a spere containing the constraint set). We will see below that one of the most successful applications of screening methods is for computing regularization paths. Given that regularization path for penalized and constrained problems coincide (up to minor details), computing the path for a penalized objective amounts to computing it for a constrained objective, whose ellipsoid initialization is safe as explained above. Even though we believe that those cases cover many (or most) problems of interest, it is also reasonable to believe that guessing the order of magnitude of the solution is feasible with simple heuristics, which is what we do for $\ell_1$-safe logistic regression. Then, it is possible to check \textit{a posteriori} that screening was safe and that indeed, the initial ellipsoid contained the solution.

\vspace*{-0.1cm}
\paragraph{Efficient implementation.}
 Since each update of the ellipsoid matrix $E$ is rank one, it is possible to parametrize $E_k$ at step $k$ as
\begin{equation*}
    E_{k} = s_k \text{I} - L_kD_kL_k^T,
\end{equation*} 
with $I$ the identity matrix, $L_k$ is in $\mathbb{R}^{p \times k}$ and $D_k$ in $\mathbb{R}^{k \times k}$ is a diagonal matrix. Hence, we only have to update $D$ and $L$ while the algorithm proceeds. 

\vspace*{-0.1cm}
\paragraph{Complexity of our screening rules.}
For each step of Algorithm~\ref{algo:ell_method}, we compute a subgradient $g$ in $O(np)$ operations. The ellipsoids are modified using rank one updates that can be stored. As a consequence, the computations at this stage are dominated by the computation of $Eg$, which is $O(pk)$. As a result, $k$ steps cost $O(k^2p + npk)$. Once we have the test set $\mathcal{X}$, we have to compute the closed forms from Lemma \ref{test} for each data point. This computation is dominated by the matrix-vector multiplications with $E$, which cost $O(kp)$ using the structure of $E$. Hence, testing the whole dataset costs $O(npk)$.
Since we typically have $n \gg k$, the cost of the overall screening procedure is therefore $O(n p k)$. In constrast, solving the ERM problem without screening would cost $O(n p T)$ where $T$ is the number of passes over the data, with $T \gg k$. With screening, the complexity becomes $O(n s T + np k)$, where $s$ is the number of data points accepted by the screening procedure. 

\subsection{Extension to Kernel Methods}

It is relatively easy to adapt our safe rules to kernel methods. Consider for example (\ref{eq:regression}), where $A$ has been replaced by $\phi(A) = [\phi(a_1),\ldots,\phi(a_n)]^\top$ in $\mathcal{H}^{n}$, with $\mathcal{H}$ a RKHS and $\phi$ its mapping function $\Real^p \rightarrow \mathcal{H}$. The prediction function $x \colon \Real^p \rightarrow \Real$ lives in the RKHS, thus it can be written $x(a) = \langle x, \phi(a) \rangle$, $\forall a \in \Real^p$. In the setting of an ERM strictly increasing with respect to the RKHS norm and each sample loss, the representer theorem ensures $x(a) = \sum_{i=1}^n \alpha_i K(a_i,a)$ with $\alpha_i \in \Real$ and $K$ the kernel associated to $\mathcal{H}$. If we consider the squared RKHS norm as the regularizer, which is typically the case, the problem becomes:
\begin{equation}
   \min_{\alpha \in \Real^n, t \in \Real^n} f(t) + \lambda \sum_{i,j = 1}^n \alpha_i \alpha_j K(a_i, a_j) \st   t = \mathbf{K}\alpha - b, \label{eq:kernelized_regression}
\end{equation}
with $\mathbf{K}$ the Gram matrix. The constraint is linear in $\alpha$ (thus satisfying to Lemma~\ref{lemma:reg}) while yielding non-linear prediction functions. The screening test becomes maximizing the linear forms $[ \mathbf{K}]_i\alpha - b_i$ and $- [\mathbf{K}]_i \alpha + b_i$ over an ellipsoid $\mathcal{X}$ containing $\alpha^*$. When the problem is convex (it depends on $K$), $\mathcal{X}$ can still be found using the ellipsoid method.

We now have an algorithm for selecting data points in regression or classification problems with linear or kernel models. As detailed above, the rules require a sparse dual, which is not the case in general except in particular instances such as support vector machines. We now explain how to induce sparsity in the dual. 

\section{CONSTRUCTING SAFE LOSSES}
\label{sec:theory}
In this section, we introduce a way to induce sparsity in the dual of empirical risk minimization problems.


\subsection{Inducing Sparsity in the Dual of ERM}

When the ERM problem does not admit a sparse dual solution, safe screening is not possible. To fix this issue, consider the ERM problem~(\ref{eq:regression}) and replace $f$ by $f_\mu$ defined in Section~\ref{sec:tools}:
\begin{equation}
   \min_{x \in \Real^p, t \in \Real^n} f_\mu(t) + \lambda R(x) \st   t = A x - b, \label{eq:regression_mod}\tag{${\mathcal P}'_1$}
\end{equation}
%
We have the following result connecting the dual of \eqref{eq:regression} with that of \eqref{eq:regression_mod}.

\begin{lemma}[Regularized dual for regression]
    The dual of \eqref{eq:regression_mod} is
\BEQ \label{dual_formula}
\max_{\nu \in \Real^n} - \langle b, \nu \rangle - f^*(\nu) - \lambda R^*\left(-\frac{A^T \nu}{\lambda}\right) - \mu \Omega(\nu),
\EEQ
and the dual of \eqref{eq:regression} is obtained by setting $\mu = 0$.
\label{lemma:reg}
\end{lemma}
The proof can be found in Appendix~\ref{sec:add_proofs}. We remark that is possible, in many cases,
to induce sparsity in the dual if $\Omega$ is the $\ell_1$-norm, or another
sparsity-inducing penalty. This is notably true if the unregularized dual is smooth with bounded gradients. In such a case, it is possible to show that the optimal dual solution would be $\nu^\star=0$ as soon as $\mu$ is large enough~\citep{bach2012optimization}.

We consider now the classification problem~(\ref{eq:classification}) and show that the previous remarks about sparsity-inducing regularization for the dual of regression problems
also hold in this new context.
\begin{lemma}[Regularized dual for classification]
Consider now the modified classification problem
\begin{equation}
   \min_{x \in \Real^p, t \in \Real^n} f_\mu(t) + \lambda R(x) \st   t=\diag(b)A x. \label{p_classif}\tag{${\mathcal P}_2'$}
\end{equation}
%
The dual of \ref{p_classif} is 
\BEQ \label{dual_formula_classif}
\max_{\nu \in \Real^n} - f^*(-\nu) - \lambda R^*\left(\frac{A^T\diag(b)\nu}{\lambda}\right) - \mu \Omega(-\nu).
\EEQ
\end{lemma}

\begin{proof}
We proceed as above with a linear constraint $\Tilde{A}\Tilde{x} = 0$ and $\Tilde{A} = (Id , - \diag(b)A)$.
\end{proof}

Note that the formula directly provides the dual of regression and classification ERM problems with a linear model such as the Lasso and SVM. 

\subsection{Link Between the Original and Regularized Problems} 
The following results should be understood as an indication that $f$ and $f_{\mu}$ are similar objectives.

 \begin{lemma}[Smoothness of $f_{\mu}$]
     If $f^* + \Omega$ is strongly convex, then $f_{\mu}$ is smooth. 
 \label{lemma:smoothness_f}
 \end{lemma}
 
 \begin{proof}
    The lemma follows directly from the fact that $f_{\mu} = (f^* + \mu \Omega)^*$ (see the proof of Lemma~\ref{lemma:reg}). The conjugate of a closed, proper, strongly convex function is indeed smooth (see \textit{e.g.}~\cite{Hiriart1993}, chapter X).
 \end{proof}
 
\begin{lemma}[Bounding the value of~\ref{eq:regression}]
\label{lemma:obj_ineq}
Let us denote the optimum objectives of \ref{eq:regression}, \ref{eq:regression_mod} by $P_{\lambda}$, $P_{\lambda, \mu}$. If $\Omega$ is a norm, we have the following inequalities:
\begin{equation*}
    P_{\lambda} - \delta^* \leq P_{\lambda, \mu} \leq P_{\lambda},
\end{equation*}
with $\delta^*$ the value of $\delta$ at the optimum of $P_{\lambda}(t) - \delta(t)$.
\end{lemma}


The proof can be found in Appendix~\ref{sec:add_proofs}. When $\mu \to 0$, $\delta(t) \to 0$ hence the objectives can be arbitrarily close.


\subsection{Effect of Regularization and Examples}
\label{subsec:examples}
We start by recalling that the infimum convolution is traditionally used for smoothing an objective when~$\Omega$ is strongly convex, and then we discuss the use of sparsity-inducing regularization in the dual.
%
%

\paragraph{Euclidean distance to a closed convex set.} It is known that convolving the indicator function of a closed convex set $\mathcal{C}$ with a quadratic term $\Omega$ (the Fenchel conjugate of a quadratic term is itself) yields the euclidean distance to $\mathcal{C}$
\begin{align*}
    f_{\mu}(t) = & \underset{z \in \mathbb{R}^n}{\text{min}} I_{\mathcal{C}}(z) + \frac{1}{2\mu}\|t - z\|_2^2 
    =  \underset{z \in \mathcal{C}}{\text{min}} \frac{1}{2\mu}\|t - z\|_2^2.
\end{align*}

\paragraph{Huber loss.} The $\ell_1$-loss is more robust to outliers than the $\ell_2$-loss, but is not differentiable in zero which may induce difficulties during the optimization. A natural solution consists in smoothing it: \cite{huber} for example show that applying the Moreau-Yosida smoothing, \textit{i.e} convolving $|t|$ with a quadratic term $\frac{1}{2} t^2$ yields the well-known Huber loss, which is both smooth and robust:
\begin{equation*} f_{\mu}(t) = \begin{cases}
               \frac{t^2}{2 \mu} & \text{if } |t| \leq \mu, \\
              |t| - \frac{\mu}{2} & \text{otherwise}.
               \end{cases}
\end{equation*}

Now, we present examples where $\Omega$ has a sparsity-inducing effect.

\paragraph{Hinge loss.} Instead of the quadratic loss in the previous example, choose a robust loss $f \colon t \mapsto \|1-t\|_1$. By using the same function $\Omega$, we obtain the classical hinge loss of support vector machines
$$
 f_\mu(t) = \sum_{i=1}^n \frac{1}{2}[1- t_i - \mu, 0]_+.
$$
We see that the effect of convolving with the constraint $\mathbf{1}_{x \preceq
0}$ is to turn a regression loss (\eg, square loss) into a classification loss.
The effect of the $\ell_1$-norm is to encourage the loss to be flat (when $\mu$ grows, $[1- t_i - \mu, 0]_+$ is equal to zero for a larger range of values $t_i$), which corresponds to the sparsity-inducing effect in the dual that we will exploit for screening data points. The Squared Hinge loss is presented in Appendix~\ref{sec:add_examples}.

\paragraph{Screening-friendly regression.}
\label{ex:sreg}
Consider now the quadratic loss $f: t \mapsto {\|t\|^2}/{2}$ and $\Omega(x) = \|x\|_1$. Then $\Omega^*(y) = {\mathbf 1}_{\|y\|_\infty \leq 1}$ (see \textit{e.g.}~\cite{bach2012optimization}), and
\begin{equation}
   f_\mu(t) = \sum_{i=1}^n \frac{1}{2}[|t_i|-\mu]_+^2. \label{eq:sreg}
\end{equation}
A proof can be found in Appendix~\ref{sec:add_proofs}. As before, the parameter $\mu$ encourages the loss to be flat (it is exactly $0$ when $\|t\|_\infty \leq \mu$).

\paragraph{Screening-friendly logistic regression.} Let us now consider the logistic loss $f(t) = \log{(1 + e^{-t})}$, which we define only with one dimension for simplicity here. It is easy to show that the infimum convolution with the $\ell_1$-norm does not induce any sparsity in the dual, because the dual of the logistic loss has unbounded gradients, making classical sparsity-inducing penalties ineffective.  However, 
we may consider instead another penalty to fix this issue: $\Omega(x) = - x \log{(-x)} + \mu |x|$ for $x \in [-1,0]$. We have $\Omega^*(y) = - e^{y + \mu - 1}$. Convolving $\Omega^*$ with~$f$ yields
\begin{equation}
    f_{\mu}(x) = \begin{cases} 
    e^{x + \mu - 1} - (x + \mu) & \: \text{if} \: x + \mu - 1 \leq 0, \\
    0 & \: \text{otherwise}.\label{eq:sclass}
    \end{cases}
\end{equation}
Note that this loss is asymptotically robust. Moreover, the entropic part of $\Omega$ makes this penalty strongly convex hence $f_{\mu}$ is smooth~\citep{Nest03}. Finally, the $\ell_1$ penalty ensures that the dual is sparse thus making the screening usable. Our regularization mechanism thus builds a smooth, robust classification loss akin to the logistic loss on which we can use screening rules. If $\mu$ is well chosen, the safe logistic loss maximizes the log-likelihood of the data for a probabilistic model which slightly differs from the sigmoid in vanilla logistic regression. The effect of regularization parameter in a few previous cases are illustrated in Figure~\ref{fig:curves}.

\begin{figure}
\centering
\begin{minipage}{0.45\linewidth}
\includegraphics[width=\linewidth]{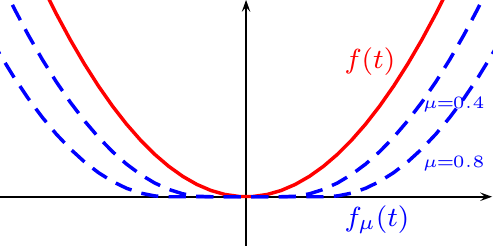}
\end{minipage} 
\begin{minipage}{0.45\linewidth}
\includegraphics[width=\linewidth]{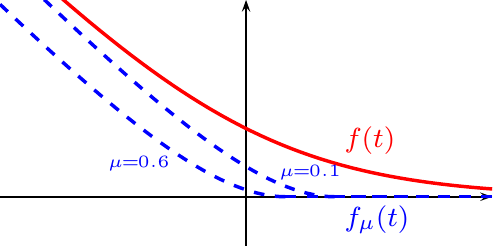}
\end{minipage}
\caption{Effect of the dual sparsity-inducing regularization on the quadratic loss~(\ref{eq:sreg}) (left) and logistic loss~(\ref{eq:sclass}) (right). After regularization, the loss functions have flat areas. Note that both of them are smooth.}\label{fig:curves}
\end{figure}

In summary, regularizing the dual with the $\ell_1$ norm induces a flat region in the loss, which induces sparsity in the dual. The geometry is preserved elsewhere. Note that we do not suggest to use~\ref{eq:regression_mod} and~\ref{p_classif} to screen for~\ref{eq:regression} and~\ref{eq:classification}.
\section{EXPERIMENTS}
\label{sec:experiments}
We now present  experimental results demonstrating the effectiveness of the data screening procedure. 

\paragraph{Datasets.}
We consider three real datasets, SVHN, MNIST, RCV-1, and a synthetic one.
MNIST ($n=60000$) and SVHN ($n=604388$) both represent digits, which we encode by using the output of a two-layer convolutional kernel network \citep{mairal2016end} leading to feature dimensions $p=2304$. RCV-1 ($n=781265$) represents sparse TF-IDF vectors of categorized newswire stories ($p=47236$). For classification, we consider a binary problem consisting of discriminating digit 9 for MNIST vs. all other digits (resp. digit 1 vs rest for SVHN, 1st category vs rest for RCV-1).
For regression, we also consider a synthetic dataset, where data is generated by  
$b = Ax + \epsilon$,
where $x$ is a random, sparse ground truth, $A \in \mathbb{R}^{n \times p}$ a data matrix whith coefficients in $[-1,1]$ and $\epsilon \sim \mathcal{N}(0, \sigma)$ with $\sigma = 0.01$. Implementation details are provided in Appendix. We fit usual models using Scikit-learn~\citep{scikit-learn} and Cyanure~\citep{Mairal2019Cyan} for large-scale datasets.

\vspace*{-0.2cm}
\subsection{Safe Screening}

Here, we consider problems that naturally admit a sparse dual solution, which allows safe screening.

\vspace*{-0.2cm}
\paragraph{Interval regression.} We first illustrate the practical use of the screening-friendly regression loss~\eqref{eq:sreg} derived above. It corresponds indeed to a particular case of a supervised learning task called interval regression \citep{hocking}, which is widely used in fields such as economics. In interval regression, one does not have scalar labels but intervals $\mathcal{S}_i$ containing the true labels $\Tilde{b}_i$, which are unknown. The loss is written
\begin{equation}
    \ell(x) = \sum_{i=1}^n \underset{b_i \in \mathcal{S}_i}{\text{inf}}(a_i^\top x - b_i)^2,
    \label{eq:general_ir}
\end{equation}
where $\mathcal{S}_i$ contains the true label $\Tilde{b}_i$. For a given data point, the model only needs to predict a value inside the interval in order not to be penalized. When the intervals $\mathcal{S}_i$ have the same width and we are given their centers $b_i$, ~\eqref{eq:general_ir} is exactly~\eqref{eq:sreg}. Since~\eqref{eq:sreg} yields a sparse dual, we can apply our rules to safely discard intervals that are assured to be matched by the optimal solution. We use an $\ell_1$ penalty along with the loss. As an illustration, the experiment was done using a toy synthetic dataset $(n = 20, p = 2)$, the signal to recover being generated by one feature only. The intervals can be visualized in Figure~\ref{fig:ir}. The ``difficult'' intervals (red) were kept in the training set. The predictions hardly fit these intervals. The ``easy'' intervals (blue) were discarded from the training set: the safe rules certify that the optimal solution will  fit these intervals. Our screening algorithm was run for 20 iterations of the Ellipsoid method. Most intervals can be ruled out afterwards while the remaining ones yield the same optimal solution as a model trained on all the intervals. 

\begin{figure}
\centering
\includegraphics[width=0.82\linewidth]{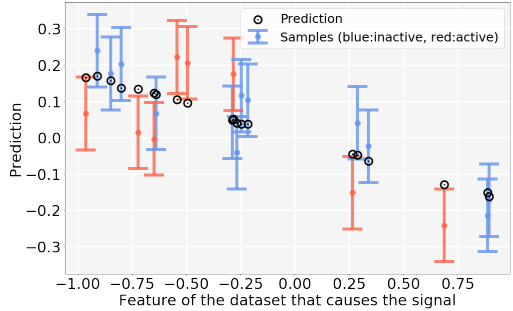}
\captionof{figure}{Safe interval regression on synthetic dataset. Most ``easy'' samples (in blue) can be discarded while the ``difficult'' ones (in red) are kept.}
\label{fig:ir}
\vspace*{-0.65cm}
\end{figure}

\vspace*{-0.2cm}
\paragraph{Classification.} 
  Common sample screening methods such as~\cite{double_screening} require a strongly convex objective. When it is not the case, there is, to the best of our knowledge, no baseline for this case. Thus, when considering classification using the non strongly convex safe logistic loss derived in Section~\ref{sec:theory} along with an $\ell_1$ penalty, our algorithm is still able to screen samples, as shown in Table~\ref{table:screened_frac_l1}. The algorithm is initialized using an approximate solution to the problem, and the radius of the initial ball is chosen depending on the number of epochs ($100$ for $10$ epochs, $10$ for $20$ and $1$ for $30$ epochs), which is valid in practice.
\begin{table*}
\vspace*{-0.2cm}
\small
\centering
\begin{tabular}{ | l | c | c | c | c | c | r | }
\hline
Epochs & \multicolumn{3}{|c|}{20} & \multicolumn{3}{|c|}{30} \\ \hline
\hline
$\lambda$ & MNIST & SVHN & RCV-1 & MNIST & SVHN & RCV-1 \\ \hline
$10^{-3}$ & 0 & 0 & 1 & 0 & 2 & 12 \\ 
$10^{-4}$ & 0.3 & 0.01 & 8 & 27 & 17 & 42 \\ 
$10^{-5}$ & 35 & 12 & 45 & 65 & 54 & 75 \\ \hline 
\end{tabular}
\caption{Percentage of samples screened (\textit{i.e} that can be thrown away) in an $\ell_1$ penalized Safe Logistic loss given the epochs made at initialization. The radius is initialized respectively at $10$ and $1$ for MNIST and SVHN at Epochs $20$ and $30$, and at $1$ and $0.1$ for RCV-1.}
\label{table:screened_frac_l1}
\vspace*{-0.2cm}
\end{table*}
\begin{table}
\vspace*{-0.2cm}
\small
\centering
\begin{tabular}{ | l | c | c | c | r | }
\hline
Epochs & \multicolumn{2}{|c|}{20} & \multicolumn{2}{|c|}{30} \\ \hline
\hline
$\lambda$ & MNIST & SVHN & MNIST & SVHN \\ \hline
$1.0 $ & 89 / 89 & 87 / 87 & 89 / 89 & 87 / 87\\ 
$10^{-1}$ & 95 / 95 & 11 / 47 & 95 / 95 & 91 / 91\\
$10^{-2}$ & 16 / 84 & 0 / 0 & 98 / 98 & 90 / 92\\ 
$10^{-3}$ & 0 / 0 & 0 / 0 & 34 / 50 & 0 / 0 \\ \hline
\end{tabular}
\caption{Percentage of samples screened in an $\ell_2$ penalized SVM with Squared Hinge loss (Ellipsoid (ours) / Duality Gap) given the epochs made at initialization.}
\label{table:screened_frac_strongly_convex}
\vspace*{-0.2cm}
\end{table}
The Squared Hinge loss allows for safe screening (see~\ref{lemma:margin_sparsity}). Combined with an $\ell_2$ penalty, the resulting ERM is strongly convex. We can therefore compare our Ellipsoid algorithm to the baseline introduced by~\cite{double_screening}, where the safe region is a ball centered in the current iterate of the solution and whose radius is  $\frac{2\Delta}{\lambda}$ with $\Delta$ a duality gap of the ERM problem. Both methods are initialized by running the default solver of scikit-learn with a certain number of epochs. The resulting approximate solution and duality gap are subsequently fed into our algorithm for initialization. Then, we perform one more epoch of the duality gap screening algorithm on the one hand, and the corresponding number of ellipsoid steps computed on a subset of the dataset on the other hand, so as to get a fair comparison in terms of data access. The results can be seen in Table~\ref{table:screened_frac_strongly_convex}. While being more general (our approach is neither restricted to classification, nor requires strong convexity), our method performs similarly to the baseline. Figure~\ref{fig:tradeoff} highlights the trade-off between optimizing and evaluating the gap (Duality Gap Screening) versus performing one step of Ellipsoid Screening. Both methods start screening after a correct iterate (i.e. with good test accuracy) is obtained by the solver (blue curve) thus 
suggesting that screening methods would rather be of practical use when computing a regularization path, or when the computing budget is less constrained (e.g. tracking or anomaly detection) which is the object of next paragraph.

\begin{figure}
\centering
\vspace*{-0.15cm}
\includegraphics[width=0.95\linewidth]{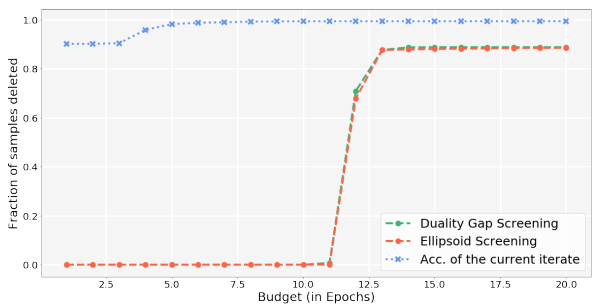}
\captionof{figure}{Fraction of samples screened vs Epochs done for two screening strategies along with test accuracy of the current iterate (Sq. Hinge + $\ell_2$ trained on MNIST).}
\label{fig:tradeoff}
\vspace*{-0.2cm}
\end{figure}

\vspace*{-0.2cm}
\paragraph{Computational gains} As demonstrated in Figure~\ref{fig:comp_gains}, computational gains can indeed be obtained in a regularization path setting (MNIST features, Squared Hinge Loss and L2 penalty). Each point of both curves represents an estimator fitted for a given lambda against the corresponding cost (in epochs). Each estimator is initialized with the solution to the previous parameter lambda. On the orange curve, the previous solution is also used to initialize a screening. In this case, the estimator is fit on the remaining samples which further accelerates the path computation.
\begin{figure}
\centering
\vspace*{-0.5cm}
\includegraphics[width=0.95\linewidth]{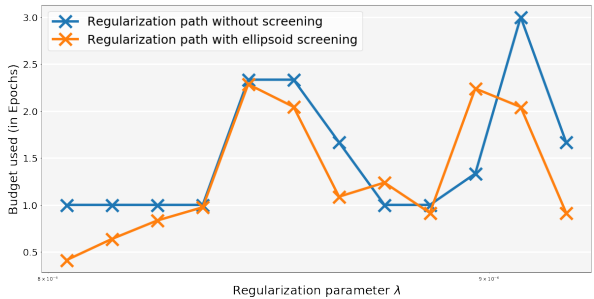}
\captionof{figure}{Regularization path of a Squared Hinge SVM trained on MNIST. The screening enables computational gains compared to a classical regularization path.}
\label{fig:comp_gains}
\vspace*{-0.5cm}
\end{figure}
\vspace*{-0.3cm}
\subsection{Dataset Compression}
We now consider the problem of dataset compression, where the goal is to maintain a good accuracy while using less examples from a dataset. This section should be seen as a proof of concept. A natural scheme consists in choosing the samples that have a higher margin since those will carry more information than samples that are easy to fit. In this setting, our screening algorithm can be used for compression by using the scores of the screening test as a way of ranking the samples. In our experiments, and for a given model, we progressively delete data points according to their score in the screening test for this model, before fitting the model on the remaining subsets. We compare those methods to random deletions in the dataset and to deletions based on the sample margin computed on early approximations of the solution when the loss admits a flat area (``margin screening''). Our compression scheme is valid for classification as can be seen in Figure~\ref{fig:compression_classif} and regression (see Appendix~\ref{sec:add_results}).

\begin{figure}
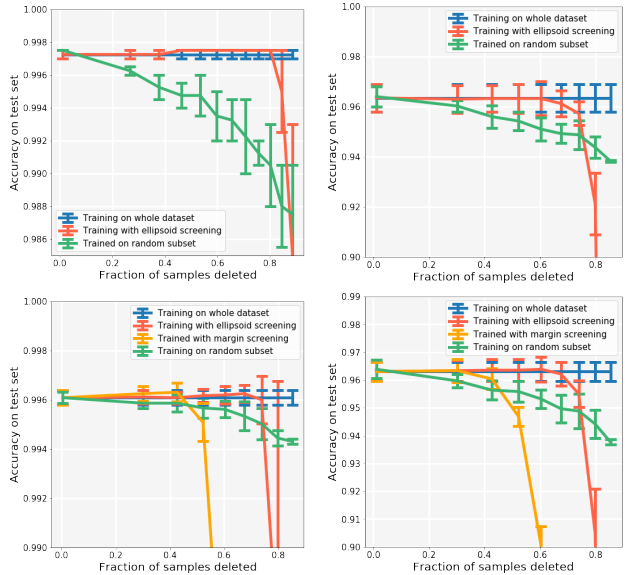

\vspace*{-0.5cm}
  \begin{subfigure}{.235\textwidth}
    \centering
    \includegraphics[width=\linewidth]{compression_mnist_safelog.pdf}
  \end{subfigure} \hfill
  \begin{subfigure}{.235\textwidth}
    \centering
    \includegraphics[width=\linewidth]{compression_svhn_safelog.pdf}
  \end{subfigure}
  
  \begin{subfigure}[t]{.235\textwidth}
    \centering
    \includegraphics[width=\linewidth]{compression_mnist_squared_hinge.pdf}
  \end{subfigure} \hfill
  \begin{subfigure}[t]{.235\textwidth}
    \centering
    \includegraphics[width=\linewidth]{compression_svhn_squared_hinge.pdf}
  \end{subfigure}
\caption{Dataset compression in classification. Up: $\ell_1$ Safe Logistic. Down: $\ell_2$ Sq. Hinge. Left: MNIST. Right: SVHN.}
  \label{fig:compression_classif}
\vspace*{-0.4cm}
\end{figure}

\paragraph{Discussion.} For all methods, the degradation in performance is lesser than with random deletions. Nevertheless, in the regime where most samples are deleted (beyond $80\%$), random deletions tend to do better. This is not surprising since the screening deletes the samples that are ``easy'' to classify. 
Then, only the difficult ones and outliers remain, making the prediction task harder compared to a random subsampling.

\clearpage

\section*{Acknowledgments}
\label{sec:acknowledgments}
JM and GM were supported by the ERC grant number 714381 (SOLARIS project) and by ANR 3IA MIAI@Grenoble Alpes, (ANR-19-P3IA-0003). AA would like to acknowledge support from the {\em ML and Optimisation} joint research initiative with the {\em fonds AXA pour la recherche} and Kamet Ventures, a Google focused award, as well as funding by the French government under management of Agence Nationale de la Recherche as part of the ``Investissements d'avenir'' program, reference ANR-19-P3IA-0001 (PRAIRIE 3IA Institute). GM thanks Vivien Cabannes, Yana Hasson and Robin Strudel for useful discussions. All the authors thank the reviewers for their useful comments.

\bibliographystyle{plainnat}
\bibliography{biblio,MainPerso}

\clearpage

\appendix

\section{Proofs.}
\label{sec:add_proofs}

%

\subsection{Proof of Lemma~\ref{lemma:margin}}
\begin{proof}
At the optimum,
\begin{align*}
    P(x^*) - D(\nu^*) ={} \frac{1}{n} \sum_{i=1}^n f_i(a_i^\top x) + f_i^*(\nu_i) + \\ \lambda R(x) + \lambda R^*\left(-\frac{A^T \nu}{\lambda n}\right) = 0.
\end{align*}
Adding the null term $\langle x, - \frac{A^\top \nu}{n} \rangle - \langle x, - \frac{A^\top \nu}{n} \rangle$ gives
\begin{align*}
    \frac{1}{n} \sum_{i=1}^n \underbrace{f_i(a_i^\top x) + f_i^*(\nu_i) - a_i^\top x \nu_i}_{\geq 0} + \\ \lambda \underbrace{ \left( R(x) + R^*\left(-\frac{A^\top \nu}{\lambda n}\right) - \left\langle x, - \frac{A^\top \nu}{\lambda n} \right\rangle \right)}_{\geq 0} = 0,
\end{align*}
since Fenchel-Young's inequality states that each term is greater or equal to zero. We have a null sum of non-negative terms; hence, each one of them is equal to zero. We therefore have for each $i = 1 \dots n$:
\begin{equation*}
    f(a_i^\top x) + f^*(\nu_i) =  a_i^\top x \nu_i,
\end{equation*}
which corresponds to the equality case in Fenchel-Young's relation, which is equivalent to $\nu^*_i \in \partial f_i(a_i^\top x^*)$. 
\end{proof}

\subsection{Proof of Lemma~\ref{lemma:test}}
\begin{proof}
\label{closed_form_optim}
 The Lagrangian of the problem writes:
\begin{align*}
    L(x, \nu, \gamma) = a_i^\top x - b_i + \nu \left( 1 - (x - z)^TE^{-1}(x - z) \right) - \\ \gamma g^T(x - z),
\end{align*}
with $\nu, \gamma \geq 0$. When maximizing in $x$, we get:
\begin{align*}
    \frac{\partial L}{\partial x} & = a_i + 2 \nu (E^{-1}z - E^{-1}x) - \gamma = 0.
\end{align*}
We have $\nu > 0$ since the opposite leads to a contradiction. This yields $x = z + \frac{1}{2 \nu}(Ea_i - \gamma Eg)$ and $(x - z)^T E^{-1} (x - z) = 1$ at the optimum which gives $\nu = \frac{1}{2}\sqrt{(a_i - \gamma)^T E (a_i - \gamma)}$. 

Now, we have to minimize
\begin{align*}
   g(\nu, \gamma) = a_i\left(z + \frac{1}{2\nu}(Ea_i - \gamma Eg)\right) - \\ \gamma^\top\left(\frac{1}{2\nu}(Ea_i - \gamma Eg)\right). 
\end{align*}
To do that, we consider the optimality condition
\begin{align*}
    \frac{\partial g}{\partial \gamma} & = - \frac{1}{2\nu}a_iEg - \frac{1}{2\nu}g^TEa_i + \frac{\gamma}{\nu} g^TEg = 0,
\end{align*}
which yields $\gamma = \frac{g^TEa_i}{g^TEg}$. If $g^TEa_i < 0$ then $\gamma = 0$ in order to avoid a contradiction.

In summary, either $g^TEa_i \leq 0$ hence the maximum is attained in $x = z + \frac{1}{2\nu}Ea_i$ and is equal to $a_iz + \sqrt{a_i^T E a_i} - y_i$, or $g^TEa_i > 0$ and the maximum is attained in $x = z + \frac{1}{2\nu}E(a_i - \gamma Eg)$ and is equal to $a_i\left(z + \frac{1}{2 \nu}E(a_i - \gamma g)\right) - b_i$ with $\nu = \frac{1}{2}\sqrt{(a_i - \gamma)^T E (a_i - \gamma)}$ and $\gamma = \frac{g^TEa_i}{g^TEg}$.
\end{proof}

\subsection{Proof of Lemma~\ref{lemma:reg}}

\begin{proof} We can write \ref{eq:regression_mod} as
\BEQ
\BA{ll}
\mbox{minimize} & \Tilde{f}(\Tilde{x}) + \lambda \Tilde{R}(\Tilde{x})\\
\mbox{subject to} & \Tilde{A}\Tilde{x} = - b
\EA\EEQ
   in the variable $\Tilde{x} = (t,x) \in \reals^{n + p}$ with $\Tilde{f} \colon \Tilde{x} \mapsto f_{\mu}(t) $ and $\Tilde{R} \colon \Tilde{x} \mapsto R(x)$ and $\Tilde{A} \in \mathbb{R}^{n \times (n+p)} = \left( \text{Id} , - A \right)$. Since the constraints are linear, we can directly express the dual of this problem in terms of the Fenchel conjugate of the objective (see \textit{e.g.} \cite{boyd_van}, 5.1.6). Let us note $f_0 = \Tilde{f} + \lambda \Tilde{R}$. For all $y \in \mathbb{R}^{n+p}$, we have
\begin{align*}
    f_0^*(y) & = \underset{x \in \mathbb{R}^{n+p}}{\text{sup}} \langle x, y \rangle - \Tilde{f}(x) - \lambda \Tilde{R}(x) \\
    & = \underset{x_1 \in \mathbb{R}^{n}, x_2 \in \mathbb{R}^p}{\text{sup}}
    \langle x_1, y_1 \rangle + \langle x_2, y_2 \rangle - f(x_1) - \lambda R(x_2) \\
    & = f_{\mu}^*(y_1) + \lambda R^*\left(\frac{y_2}{\lambda}\right).
\end{align*}
It is known from~\cite{huber} that $f_{\mu} = f ~\square~ \Omega^*_{\mu} = (f^* + \Omega_{\mu}^{**})^*$ with $\Omega_{\mu}^* = \mu \Omega^*(\frac{.}{\mu})$. Clearly, $\Omega_{\mu}^{**} = \mu \Omega$. If $\Omega$ is proper, convex and lower semicontinuous, then $\Omega = \Omega^{**}$ . As a consequence, $f_{\mu}^* = (f^* + \mu \Omega)^{**}$. If $f^* + \mu \Omega$ is proper, convex and lower semicontinuous, then $f_{\mu}^* = f^* + \mu \Omega$, hence
\begin{equation*}
    f_0^*(y) = f^*(y_1) + \lambda R^*\left(\frac{y_2}{\lambda}\right) + \mu \Omega(y_1).
\end{equation*}
Now we can form the dual of \ref{eq:regression_mod} by writing 
\BEQ
\BA{ll}
\mbox{maximize} & - \langle - b, \nu \rangle - f_0^*(-\Tilde{A}^T\nu)
\EA\EEQ
in the variable $\nu \in \mathbb{R}^n$. Since $-\Tilde{A}^T \nu = (-\nu, A^T \nu)$ with $\nu \in \mathbb{R}^n$ the dual variable associated to the equality constraints,
\[
    f_0^*(-\Tilde{A}^T \nu) = f^*(-\nu) + \lambda R^*\left(\frac{A^T \nu}{\lambda}\right) + \mu \Omega(-\nu).
\]
Injecting $f_0^*$ in the problem and setting $\nu$ instead of $-\nu$ (we optimize in $\mathbb{R}$) concludes the proof.
\end{proof}

\subsection{Lemma~\ref{lemma:bounding_f}}

 \begin{lemma}[Bounding $f_{\mu}$]
      If $\mu \geq 0$ and $\Omega$ is a norm then
      \begin{equation*}
          f(t) - \delta(t) \leq f_{\mu}(t) \leq f(t),\quad \mbox{for all $t \in \mathrm{dom} f$}
      \end{equation*}
 with  $\delta(t) = \underset{\|\frac{u}{\mu}\|^* \leq 1}{\max} g^Tu$ and $g \in \partial f(t)$.
 \label{lemma:bounding_f}
 \end{lemma}

\begin{proof}
    If $\Omega$ is a norm, then $\Omega(0)=0$ and $\Omega^*$ is the indicator function of the dual norm of $\Omega$ hence non-negative. Moreover, if $\mu > 0$ then, $\forall z \in \text{dom}f$ and $\forall t \in \mathbb{R}^n$,
    \begin{equation*}
        f_{\mu}(t) \leq f(z) + \mu \Omega^*\left(\frac{t - z}{\mu}\right).
    \end{equation*}
    In particular, we can take $t = z$ hence the right-hand inequality. On the other hand,
    \begin{align*}
        f_{\mu}(t) - f(t) &= \underset{z}{\min} f(z) + \mu I_{\|\frac{z - t}{\mu}\|^* \leq 1} - f(t)\\
        & = \underset{\|\frac{u}{\mu}\|^* \leq 1}{\min} f(t + u) - f(t).
    \end{align*}
    Since $f$ is convex,
    \begin{equation*}
        f(t + u) - f(t) \geq g^Tu \text{ with } g \in \partial f(t).
    \end{equation*}
    As a consequence, 
    \begin{equation*}
        f_{\mu}(t) - f(t) \geq \underset{\|\frac{u}{\mu}\|^* \leq 1}{\min} g^Tu.
    \end{equation*}
\end{proof}

\subsection{Proof of Lemma~\ref{lemma:obj_ineq}}

\begin{proof}
   The proof is trivial given the inequalities in Lemma~\ref{lemma:bounding_f}.
\end{proof}

\subsection{Proof of Screening-friendly regression}

\begin{proof}   The Fenchel conjugate of a norm is the indicator function of the unit ball of its dual norm, the $\ell_\infty$ ball here. Hence the infimum convolution to solve
\begin{equation} \label{eq:lasso_loss_modified}
    f_{\mu}(x) = \underset{z \in \mathbb{R}^n}{\text{min }} \{f(z) + \mathbf{1}_{\|x - z\|_{\infty} \leq \mu}\}
\end{equation}
Since $f(x) = \frac{1}{2n} \|x\|_2^2$,
\begin{equation*}
    f_{\mu}(x) = \underset{z \in \mathbb{R}^{n}}{\text{min }} \frac{1}{2n} z^Tz + \mathbf{1}_{\|x - z\|_{\infty} \leq \mu}.
\end{equation*}
If we consider the change of variable $t = x - z$, we get:
\begin{equation*}
    f_{\mu}(x) = \underset{t \in \mathbb{R}^n}{\text{min }} \frac{1}{2n} \|x - t\|_2^2 + \mathbf{1}_{\|t\|_{\infty} \leq \mu}.
\end{equation*}
The solution $t^*$ to this problem is exactly the proximal operator for the indicator function of the infinity ball applied to $x$. It has a closed form
\begin{align*}
t^* & = \text{prox}_{\mathbf{1}_{\|.\|_{\infty} \leq \mu}}(x) \\
    & = x - \text{prox}_{\left(\mathbf{1}_{\|.\|_{\infty} \leq \mu}\right)^*}(x),
\end{align*}
using Moreau decomposition. We therefore have
\begin{align*}
t^* & = x - \text{prox}_{\mu \|.\|_1}(x). 
\end{align*}
Hence,
\[
    f_{\mu}(x) = \frac{1}{2n} \| x - t^*\|_2^2 = \frac{1}{2n} \| \text{prox}_ {\mu \|.\|_1}(x)\|_2^2.
\]
But, $\text{prox}_ {\mu \|.\|_1}(t) = \text{sgn}(t) \times [|t| - \mu]_+$ for $t \in \mathbb{R}$, where $[x]_+ = \text{max}(x, 0)$. 
\end{proof}

\section{Additional examples.}
\label{sec:add_examples}
\paragraph{Squared hinge loss.} Let us consider a problem with a quadratic loss $f \colon t \mapsto \| 1-t\|_2^2/2$ designed for a classification problem, and consider
$\Omega(x)= \|x\|_1 + \mathbf{1}_{x \preceq 0}$. We have $\Omega^*(y) = \mathbf{1}_{y \succeq -1}$, and 
\begin{align*}
    f_{\mu}(t) = & 
     \sum_{i=1}^n [1- t_i - \mu, 0]_+^2,
\end{align*}
which is a squared Hinge Loss with a threshold parameter~$\mu$ and $[.]_+ = \max(0,.)$.
 
\section{Additional experimental results.}
\label{sec:add_results}
\paragraph{Reproducibility.}  The data sets did not require any pre-processing except \emph{MNIST} and \emph{SVHN} on which exhaustive details can be found in \cite{mairal2016end}. For both regression and classification, the examples were allocated to train and test sets using scikit-learn's \textit{train-test-split} ($80\%$ of the data allocated to the train set). The experiments were run three to ten times (depending on the cost of the computations) and our error bars reflect the standard deviation. For each fraction of points deleted, we fit three to five estimators on the screened dataset and the random subset before averaging the corresponding scores. The optimal parameters for the linear models were found using a simple grid-search. 

\paragraph{Accuracy of our safe logistic loss.} 

The accuracies of the Safe Logistic loss we build is similar to the accuracies obtained with the Squared Hinge and the Logistic losses on the datasets we use in this paper thus making it a realistic loss function.

\begin{table*}
\small
\centering
\begin{tabular}{ | l | c | c | c |  }
\hline
Dataset & MNIST & SVHN & RCV-1 \\ \hline
\hline
Logistic + $\ell_1$ & 0.997 (0.01) & 0.99 (0.0003) & 0.975 (1.0) \\
Logistic + $\ell_2$ & 0.997 (0.001) & 0.99 (0.0003) & 0.975 (1.0) \\
Safelog + $\ell_1$ & 0.996 (0.0) & 0.989 (0.0) & 0.974 (1e-05) \\
Safelog + $\ell_2$ & 0.996 (0.0) & 0.989 (0.0) & 0.975 (1e-05) \\
Squared Hinge + $\ell_1$ & 0.997 (0.03) & 0.99 (0.03) & 0.975 (1.0) \\
Squared Hinge + $\ell_2$ & 0.997 (0.003) & 0.99 (0.003) & 0.974 (1.0) \\
\hline
\end{tabular}
\caption{Averaged best accuracies on test set (best $\lambda$ in a Logarithmic grid from $\lambda=0.00001$ to $1.0$).}
\label{table:accuracies}
\end{table*}

\paragraph{RCV-1.} Table~\ref{table:safe_sqhinge_rcv1} shows additional screening results on RCV-1 with a $\ell_2$ penalized Squared Hinge loss SVM.

\begin{table}[H]
\small
\centering
\begin{tabular}{ | l | c | r | }
\hline
Epochs & 10 & 20\\ \hline
\hline
$\lambda = 1$ & 7 / 84 & 85 / 85 \\ 
$\lambda = 10$ & 80 / 80 & 80 / 80 \\ 
$\lambda = 100$ & 68 / 68 & 68 / 68 \\ \hline 
\end{tabular}
\caption{RCV-1 : Percentage of samples screened in an $\ell_2$ penalized SVM with Squared Hinge loss (Ellipsoid (ours) /
Duality Gap) given the epochs made at initialization.}
\label{table:safe_sqhinge_rcv1}
\end{table}

\paragraph{Lasso regression.} The Lasso objective combines an $\ell_2$ loss with an $\ell_1$ penalty.
Since its dual is not sparse, we will instead apply the safe rules offered by the screening-friendly regression loss~\eqref{eq:sreg} derived in Section \ref{subsec:examples} and illustrated in~Figure~\ref{fig:curves}, combined with an $\ell_1$ penalty.
We can draw an interesting parallel with the SVM, which is naturally sparse in data points. At the optimum, the solution of the SVM can be expressed in terms of data points (the so-called support vectors) that are close to the classification boundary, that is the points that are \textit{the most difficult} to classify. Our screening rule yields the analog for regression: the points that are easy to predict, i.e. that are close to the regression curve, are less informative than the points that are harder to predict. 
In our experiments on \emph{synthetic data} ($n=100$), this does consistently better than random subsampling as can be seen in Figure~\ref{fig:synthetic_compression}.  

\begin{figure}
\centering
  \includegraphics[width=0.9\linewidth]{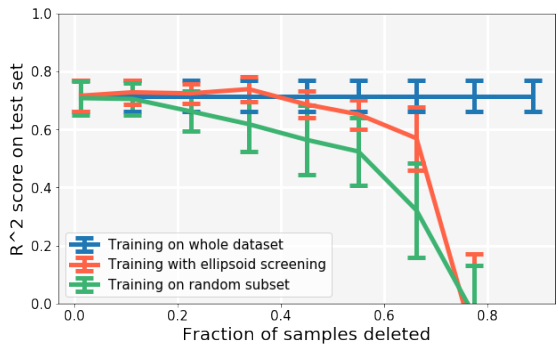}
  \captionof{figure}{Dataset compression for the Lasso trained on a synthetic dataset. The scores given by the screening yield a ranking that is better than random subsampling.}
  \label{fig:synthetic_compression}
\end{figure}

\begin{figure}
\centering
  \begin{subfigure}{.4\textwidth}
    \centering
    \includegraphics[width=\linewidth]{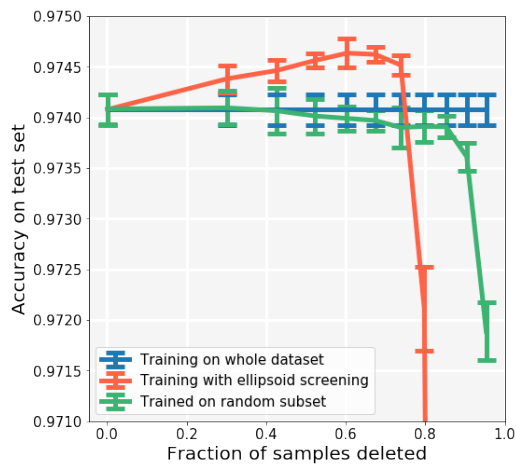}
    \caption{RCV-1 and $\ell_1$ Safe Logistic}
  \end{subfigure}
    
  \begin{subfigure}{.4\textwidth}
    \centering
    \includegraphics[width=\linewidth]{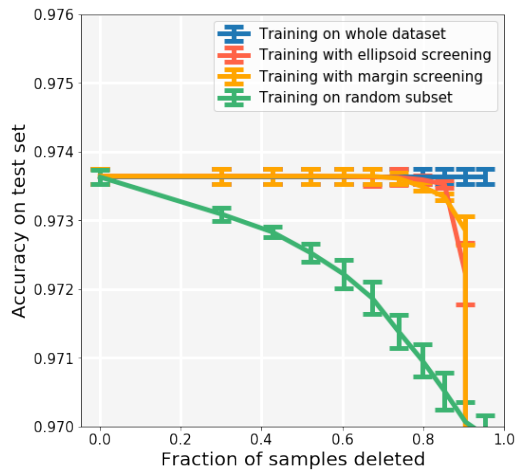}
    \caption{RCV-1 and $\ell_2$ Squared Hinge}
  \end{subfigure}
\caption{Dataset compression in classification.}
  \label{fig:compression_classif_hinge}
\end{figure}

%
%

\end{document}